\documentclass[letterpaper,10pt,conference]{ieeeconf}

\IEEEoverridecommandlockouts
\usepackage[utf8]{inputenc}
\usepackage[T1]{fontenc}
\usepackage{microtype}
\usepackage{graphicx}
\usepackage{booktabs}
\usepackage{multirow}
\usepackage{amsmath}
\usepackage{amssymb}
\usepackage{amsfonts}
\usepackage{mathtools}
\usepackage[table]{xcolor}
\usepackage{url}
\usepackage{xspace}
\usepackage{bbm}

\usepackage{caption}
\usepackage{subcaption}
\usepackage{rotating}
\usepackage{pdflscape}

\usepackage{algorithm}
\usepackage[noend]{algpseudocode}

\usepackage[colorlinks=true,linkcolor=blue,citecolor=blue,urlcolor=blue]{hyperref}
\usepackage[capitalize,noabbrev]{cleveref}
\providecommand{\citep}[1]{\cite{#1}}
\providecommand{\citet}[1]{\cite{#1}}

\newtheorem{theorem}{Theorem}[section]
\newtheorem{proposition}[theorem]{Proposition}

\makeatletter
\if@confmode
\def\section{\@startsection{section}{1}{\z@}{1.5ex plus 1.5ex minus 0.5ex}%
{0.7ex plus 1ex minus 0ex}{\normalfont\normalsize\centering\bfseries\scshape}}%
\def\subsection{\@startsection{subsection}{2}{\z@}{1.5ex plus 1.5ex minus 0.5ex}%
{0.7ex plus .5ex minus 0ex}{\normalfont\normalsize\bfseries}}%
\fi
\makeatother

\definecolor{cellgreen}{RGB}{200,235,200}
\definecolor{cellred}{RGB}{255,210,210}

\newcommand{\E}{\mathbb{E}}
\newcommand{\Var}{\operatorname{Var}}

\title{\LARGE \bf
Learning from the Best: Smoothness-Driven Metrics\\ 
for Data Quality in Imitation Learning
}

\author{
Soham Kulkarni$^{1}$, Raayan Dhar$^{1}$, Yuchen Cui$^{1}$%
\thanks{$^{1}$Soham Kulkarni, Raayan Dhar, and Yuchen Cui are with the Department of Computer Science,
University of California, Los Angeles, CA 90095, USA.
{\tt\small \{sohamkulkarni,yuchencui\}@cs.ucla.edu, raayan.dhar@gmail.com}}
}

\begin{document}

\maketitle
\thispagestyle{empty}
\pagestyle{empty}

\begin{abstract}
In behavioral cloning (BC), policy performance is fundamentally limited by demonstration data quality. Real-world datasets contain trajectories of varying quality due to operator skill differences, teleoperation artifacts, and procedural inconsistencies---yet standard BC treats all demonstrations equally. Existing curation methods require costly policy training in the loop or manual annotation, limiting scalability.
We propose \textsc{rinse} (Ranking and INdexing Smooth Examples), a lightweight framework for scoring demonstrations based on trajectory smoothness that is policy-architecture-agnostic and operates on trajectory data alone, with TED additionally using a phase-boundary/contact signal. Grounded in motor control theory---which establishes smoothness as a hallmark of skilled movement---\textsc{rinse} uses two complementary metrics: Spectral Arc Length (SAL), a spectral measure of frequency-domain regularity, and Trajectory-Envelope Distance (TED), a spatial measure of contact-aware geometric deviation. We show that smoothness filtering can reduce the conditional action variance of the retained data distribution, with downstream effects that can be amplified by action chunking and compounding error.
On RoboMimic benchmarks, SAL filtering achieves 16\% higher success using one-sixth of the data. On real-world manipulation, TED filtering achieves 20\% improvement with half the data. As a retrieval-stage filter within STRAP on LIBERO-10, \textsc{rinse} re-ranking improves mean success by 5.6\%. As soft weights in Re-Mix domain reweighting, \textsc{rinse} scores produce domain allocations highly correlated with the learned Re-Mix allocations (Spearman $\rho \ge 0.89$). These results support smoothness as a useful quality signal across filtering, retrieval, and reweighting settings, especially in noisy or heterogeneous data regimes.
\end{abstract}

\section{Introduction}
\label{sec:intro}

Behavioral cloning (BC), training a policy to mimic expert demonstrations via supervised learning, remains the dominant paradigm for teaching robots manipulation skills~\cite{pomerleau1989alvinn,ross2010efficient}. Its simplicity has made it the backbone of modern robot learning systems, from single-task policies~\cite{chi2023diffusionpolicy,zhao2023learning} to generalist models~\cite{octo_2023}. Yet a persistent bottleneck limits its effectiveness: the quality of training data.

Real-world demonstration datasets are inherently heterogeneous. Human operators vary in skill, fatigue, and strategy; teleoperation interfaces introduce device-specific artifacts; and large-scale collection campaigns amplify these inconsistencies~\cite{robomimic2021}. Standard BC treats all demonstrations equally, minimizing a per-step loss under the assumption that every trajectory is equally informative. When noisy or suboptimal demonstrations inflate the \emph{conditional action variance}, the irreducible noise floor of the training objective, the resulting policies learn averaged, hesitant behaviors rather than decisive, expert-like motions.

Recent work has begun to address data curation for robot learning. Learning-based methods score demonstrations using influence functions~\cite{agia2025cupid}, mutual information~\cite{Hejna2025}, or datamodels~\cite{dass2025datamil}, but these require policy training in the loop or expensive computations over the full dataset. Heuristic filters based on task success or trajectory length~\cite{robomimic2021,Ahn2022} fail to distinguish quality differences among successful demonstrations. There is a clear need for metrics that are \emph{lightweight}, \emph{offline}, and usable before any policy is trained, while remaining compatible with different policy architectures.

We propose \textsc{rinse} (Ranking and INdexing Smooth Examples), a framework that filters demonstrations based on trajectory smoothness. Our approach is grounded in motor control theory: skilled human movements are smooth because the motor system minimizes jerk and effort~\cite{flash1985coordination,harris1998signal}. Noisy execution, whether from tremor, hesitation, or device artifacts, manifests as high-frequency energy in the trajectory. Smoothness therefore serves as a principled proxy for demonstration quality, rooted in the biomechanics of skilled behavior rather than arbitrary heuristics.

\textsc{rinse} employs two complementary metrics: \textbf{Spectral Arc Length (SAL)}, a clinically validated smoothness metric from rehabilitation science~\cite{Balasubramanian2012SPARC} that we adapt for robot data curation, and \textbf{Trajectory-Envelope Distance (TED)}, a novel spatial metric that measures geometric deviation from a contact-aware B\'ezier-smoothed reference in Cartesian space. Both metrics are reference-free, require no policy training, and run in sub-second time per demonstration in our tested settings. We provide a theoretical analysis showing how smoothness filtering can reduce the conditional action variance of the retained data distribution, with downstream effects that can be amplified by action chunking and compounding error.

We validate \textsc{rinse} in two settings: (1) as a quality diagnostic that ranks datasets and domains without policy training (RTIS modality benchmark~\cite{li2025how}, Re-Mix~\cite{hejnaremix} domain reweighting), and (2) as a filter for policy learning (RoboMimic~\cite{robomimic2021}, real-world tasks, STRAP~\cite{memmel2024strap} retrieval on LIBERO-10~\cite{liu2024libero}). 
Our results show that \textsc{rinse} scores correlate well with both data collection modality quality~\cite{li2025how} and learned domain allocations from Re-Mix.
Furthermore, SAL filtering achieves a 16\% higher success rate on Transport using 50 out of 300 demonstrations; TED filtering achieves 20\% improvement on real-world Push Block with half the data; and \textsc{rinse} re-ranking improves the success rate of STRAP retrieval by 5.6\% on LIBERO-10. 

Our contributions are: (1)~adaptation of Spectral Arc Length (SAL), a clinically validated smoothness metric from rehabilitation science, to robot demonstration curation, paired with a novel spatial metric, Trajectory-Envelope Distance (TED), that captures contact-aware geometric quality; (2)~a theoretical analysis connecting smoothness filtering to conditional action variance in BC and its downstream impact through compounding error; and (3)~empirical validation across standalone filtering, retrieval re-ranking, and domain reweighting, showing that \textsc{rinse} serves as a versatile plug-in quality signal that often matches or outperforms full-data training with substantially fewer demonstrations.

\begin{figure*}[t]
  \centering
  \includegraphics[width=\linewidth]{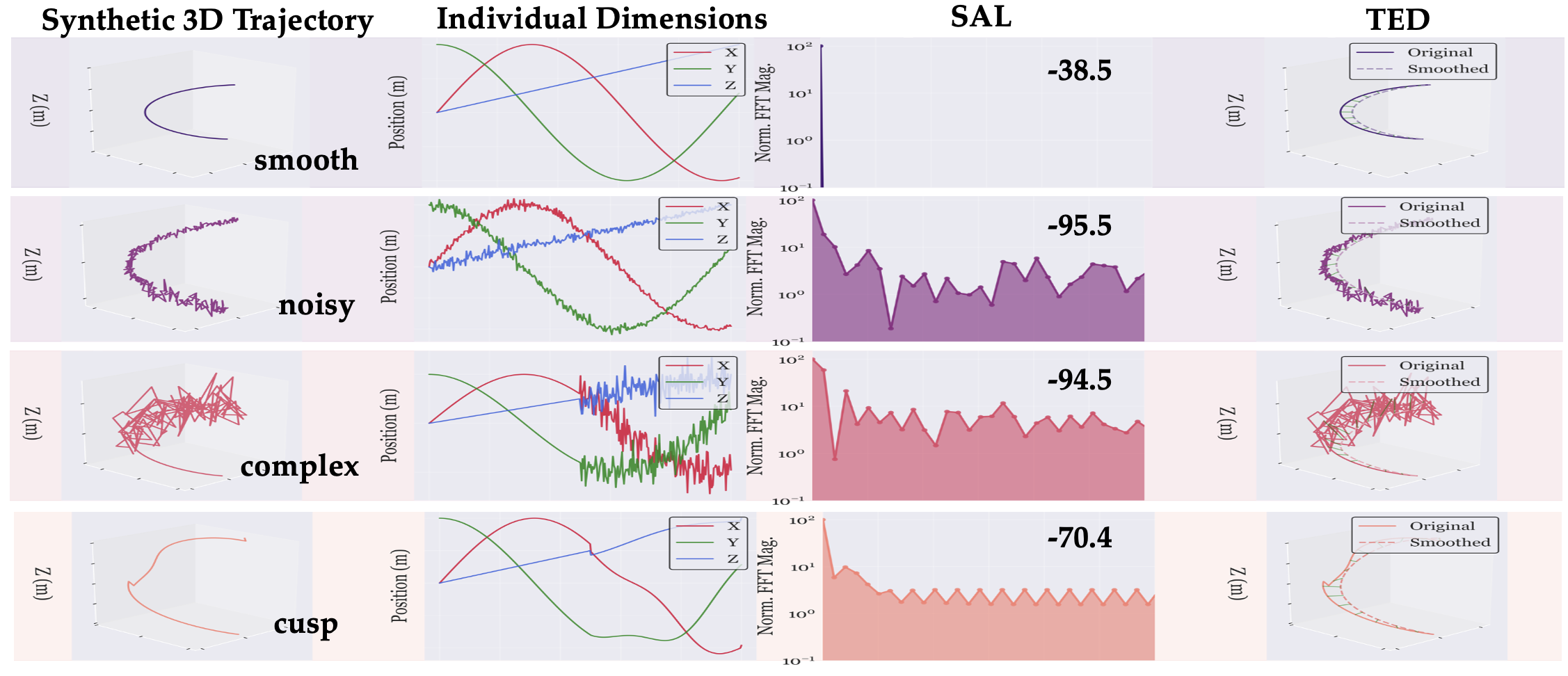}
  \caption{Smoothness metrics applied to synthetic trajectories. Left to right: 3D trajectory, individual axes, SAL spectral analysis, and TED B\'ezier curve fitting. Smooth trajectories (top) have steep spectral decay and small B\'ezier residuals; noisy trajectories (bottom) show flat spectra and large deviations.}
  \label{fig:methods_viz}
\end{figure*}

\section{Related Work}
\label{sec:relatedwork}

Data quality control for robot learning has gained growing attention as demonstration datasets grow in scale. We review key approaches and position \textsc{rinse} within this landscape.

\paragraph{Learning-based curation.}
A growing family of methods assess demonstration quality through learned representations or policy feedback. Influence-function approaches such as CUPID~\cite{agia2025cupid} rank demonstrations by estimated effect on expected return, achieving strong results but requiring policy rollouts and inverse-Hessian computations. DemInf~\cite{Hejna2025} estimates per-trajectory mutual information $I(S;A)$ to rank informativeness, requiring dataset-global statistics and VAE embeddings. DataMIL~\cite{dass2025datamil} trains many models on random subsets to compute data values, capturing interaction effects at high computational cost. Re-Mix~\cite{hejnaremix} performs distributionally robust optimization over domain-level mixing weights. STRAP~\cite{memmel2024strap} retrieves relevant sub-trajectories from a prior pool using subsequence DTW matching. Other methods leverage online rollout signals~\cite{chen2025demoscore}, self-supervised embeddings~\cite{zhang2025scizor}, or importance-weighted retrieval~\cite{du2023behavior}. While these methods can maximize data utility, they introduce computational overhead and often couple quality assessment with policy learning.

\paragraph{Heuristic-based filtering.}
In practice, most datasets are pre-filtered using human-defined rules: task success, trajectory length, or physical checks such as collision counts~\cite{robomimic2021,Ahn2022,shukla2024maniskill}. SayCan~\cite{Ahn2022} collected 276{,}000 autonomous rollouts and retained only 12{,}000 via success filtering and human review, underscoring the data wastage inherent in coarse heuristics. These heuristics capture coarse quality differences but miss subtle execution fidelity: among successful demonstrations, some are smooth and decisive while others are jittery and hesitant. The key challenge is distinguishing expert-like behavior from noisy execution within the set of task-successful trajectories.

\paragraph{Data quality theory for BC.}
Belkhale et al.~\cite{belkhale2024data} formalized data quality for imitation learning through action consistency and state diversity, showing that both axes independently affect downstream policy performance. Lin et al.~\cite{lin2024data} established scaling laws for manipulation, finding that environmental diversity matters more than raw quantity. The classical BC regret bound $J(\hat\pi) - J(\pi^*) \le H^2\varepsilon$~\cite{ross2010efficient,rajaraman2020fundamental} is tight under MSE loss, and recent work by Foster et al.~\cite{foster2024horizon} shows that horizon-independent bounds are achievable under log-loss but remain elusive for MSE objectives used in standard practice. These results motivate quality-focused approaches: since the $H^2$ compounding cannot be eliminated algorithmically for MSE-based BC, reducing the per-step error $\varepsilon$ through better data is the most direct path to improvement.

\paragraph{Smoothness in motor control.}
Trajectory smoothness has a long history as a quality indicator. The minimum-jerk model~\cite{flash1985coordination} predicts that skilled reaching follows fifth-order polynomials with bell-shaped velocity profiles. Signal-dependent noise~\cite{harris1998signal} explains why noisy execution is inherently non-smooth. In rehabilitation, Spectral Arc Length (SPARC)~\cite{Balasubramanian2012SPARC} has been validated as the most reliable smoothness metric with intraclass correlation (ICC) exceeding 0.9~\cite{balasubramanian2015analysis,cornec2024sparc}. Sakr et al.~\cite{sakr2025consistency} showed that smoothness-based consistency metrics predict 70--89\% of task success rates in robot learning. \textsc{rinse} brings these validated metrics to data curation, providing a quality signal grounded in decades of motor control research.

\paragraph{Positioning \textsc{rinse}.}
\textsc{rinse} requires no policy training, rollouts, or dataset-global statistics, and is broadly compatible with different tasks and policy architectures; TED additionally uses a phase-boundary/contact signal. It provides a fast, principled first pass that is complementary to more expensive methods.

\section{RINSE: Ranking and INdexing Smooth Examples}
\label{sec:method}

We present the \textsc{rinse} framework for evaluating and filtering demonstration quality based on trajectory smoothness. We begin with the BC problem setup and a key decomposition of the training loss (\S\ref{sec:setup}), introduce two smoothness metrics (\S\ref{sec:sal-definition}--\ref{sec:ted-main}), and provide theoretical motivation for why smoothness filtering improves BC (\S\ref{sec:theory}).

\subsection{Problem Setup and the Noise Floor}
\label{sec:setup}

We consider a standard imitation learning setting with continuous state space $\mathcal{S} \subset \mathbb{R}^{d_s}$ and action space $\mathcal{A} \subset \mathbb{R}^{d_a}$. A dataset $\mathcal{D} = \{(s_t^{(i)}, a_t^{(i)})\}$ with $t$ denote the timestep inside the $i$-th ($ i \in [0, N)$) demonstration  out of a total of $N$ expert demonstrations. $m = \|\mathcal{D}\|$ is the number of total state-action pairs. $\mathcal{D}$ is used to train a parametric policy $\pi_\theta: \mathcal{S} \to \mathcal{A}$ via empirical risk minimization:
\begin{equation}
    \hat{\theta} = \arg\min_\theta \frac{1}{m}\sum_{(s,a) \in \mathcal{D}} \|\pi_\theta(s) - a\|_2^2.
    \label{eq:erm}
\end{equation}
The population risk is $L(\theta) = \mathbb{E}_{(s,a)\sim p_{\text{demo}}}[\|\pi_\theta(s) - a\|_2^2]$.

\smallskip
\noindent\begin{proposition}[\textbf{Noise floor decomposition}]
\label{prop:noise_floor}
For any policy $\pi_\theta$, the population risk decomposes as:
\begin{equation}
    L(\theta) = \underbrace{\E_s\bigl[\|\pi_\theta(s) - \mu^*(s)\|^2\bigr]}_{\text{policy error}} \;+\; \underbrace{\E_s\bigl[\operatorname{Tr}\bigl(\Var[a \mid s]\bigr)\bigr]}_{\bar\sigma^2\;\text{(noise floor)}},
    \label{eq:noise_floor}
\end{equation}
where $\mu^*(s) = \E[a \mid s]$ is the conditional mean action.
\end{proposition}

\begin{proof}
Expanding $\|\pi_\theta(s) - a\|^2 = \|\pi_\theta(s) - \mu^*(s) + \mu^*(s) - a\|^2$, the cross-term $2\langle\pi_\theta(s) - \mu^*(s),\, \mu^*(s) - a\rangle$ vanishes under $\E[\cdot \mid s]$ because $\E[a \mid s] = \mu^*(s)$, yielding the decomposition by the law of total variance.
\end{proof}

\noindent We write the decomposition in terms of states for simplicity; in practice, the same argument applies to observations or local state neighborhoods, where partial observability and local aliasing induce action inconsistency among nearby inputs.

\noindent The noise floor $\bar\sigma^2$ is the irreducible per-state action inconsistency; it cannot be reduced by \emph{any} policy. When demonstrations come from operators of varying skill, $\bar\sigma^2$ increases because different operators can produce different actions for the same observation or nearby aliased states. The first term, policy error, is what training can minimize.

This decomposition is the theoretical foundation for \textsc{rinse}: by filtering to retain only smooth, high-quality demonstrations, we can reduce the conditional action variance of the retained data distribution, which (i)~lowers the loss floor on that retained subset, (ii)~reduces gradient variance during training (since gradient noise increases with the residual magnitude), and (iii)~produces a cleaner learning signal.

\subsection{Spectral Arc Length (SAL)}
\label{sec:sal-definition}

SAL~\cite{Balasubramanian2012SPARC} quantifies trajectory smoothness in the frequency domain. From each demonstration, we extract a scalar speed signal $v[0{:}T{-}1]$ (e.g., the norm of end-effector velocity, sampled at interval $\Delta t$). Let $V_k = \hat{v}(f_k)$ be its one-sided Fast Fourier Transform (FFT) at frequencies $f_k = k/(T\Delta t)$, with log-amplitude $L_k = \log(|V_k| + \varepsilon)$. The spectral arc length is:
\begin{equation}
    \mathrm{SAL}(v) := -\sum_{k=1}^{T/2} \sqrt{(f_k - f_{k-1})^2 + (L_k - L_{k-1})^2}.
    \label{eq:sal}
\end{equation}
A less-negative SAL score (closer to zero) indicates smoother motion. Each summand is the Euclidean line element along the log-spectrum curve; a short total arc length forces the spectrum to decay steeply, penalizing high-frequency energy. SAL is scale-invariant and computable in $O(T\log T)$ (see Figure~\ref{fig:methods_viz}).

\subsection{Trajectory-Envelope Distance (TED)}
\label{sec:ted-main}

SAL captures global frequency characteristics but ignores Cartesian geometry and the local structure of contact events. We introduce the \emph{Trajectory-Envelope Distance} (TED), a self-similarity measure quantifying how much distortion is needed to deform a raw demonstration into its smoothest feasible version while respecting contact constraints.

\paragraph{Contact-aware partition.}
We consider the proprioceptive end-effector states of a demonstration trajectory $\tau = \{(x_t, R_t)\}_{t=0}^{T}$ with positions $x_t \in \mathbb{R}^3$ and orientations $R_t \in \mathrm{SO}(3)$ is partitioned by gripper contact. Each maximal sub-interval of constant contact status is split into \emph{dense} boundary regions (the first and last $r$ fraction) and a \emph{sparse} interior, with separate smoothing hyperparameters applied to the boundary and interior portions.

\paragraph{B\'ezier envelope.}
For each region, we fit a low-control-point B\'ezier curve using a greedy refinement procedure (see Algorithm~\ref{alg:ted} and \hyperref[appendix:ted]{Appendix~b}) that iteratively inserts control points at locations of maximum residual, keeps the best-scoring envelope under an area-plus-corridor score, and falls back to corridor-based subdivision when needed. B\'ezier curves of degree $K \ll T$ cannot represent high-frequency oscillations by construction, making the residual a natural measure of non-smooth content. Boundary-region blending can optionally anchor dense regions more closely to the raw trajectory. Orientations are smoothed in rotation-vector form using sliding-window geodesic averaging (Karcher mean) or related interpolation variants.

\paragraph{Normalized DTW distance.}
After concatenating position $x_t \in \mathbb{R}^3$ and rotation-vector orientation $r_t \in \mathbb{R}^3$ into $z_t = [x_t^\top, r_t^\top]^\top \in \mathbb{R}^6$, where $r_t$ parameterizes $R_t$, the per-pair cost
$
  d(z_t, \tilde{z}_s) = ({\|x_t - \tilde{x}_s\|^2 + w_{\mathrm{ori}}^2\,\theta(R_t, \tilde{R}_s)^2})^{1/2}
$
is fed to FastDTW with alignment path $P$. The TED score is:
\begin{equation}
    d_{\mathrm{TED}}(\tau) = \frac{1}{|P|}\sum_{(t,s) \in P} d(z_t, \tilde{z}_s).
    \label{eq:ted}
\end{equation}
Lower TED indicates higher intrinsic smoothness. The B\'ezier envelope acts as an implicit low-pass filter: since low-degree polynomial curves cannot represent high-frequency oscillations, the residual between the raw trajectory and its envelope is dominated by high-frequency content. Consequently, a small TED score indicates that the trajectory already has limited high-frequency energy, consistent with a steep spectral decay (large SAL). TED is computed in $O(KT)$ (FastDTW is linear in $T$), typically tens of milliseconds per trajectory. Critically, TED's contact-aware partition avoids penalizing natural discontinuities at contact transitions, a key advantage over global metrics for manipulation tasks.

\subsection{Theoretical Motivation}
\label{sec:theory}

We now trace the theoretical case for smoothness-based filtering, building on established results in motor control and learning theory. Rather than attempting to prove that trajectory smoothness directly controls the policy's function class (a connection that requires strong assumptions on the dynamics mapping~\cite{ross2010efficient}), we identify a cleaner mechanism through the noise floor.

\paragraph{Smoothness indicates skilled execution.}
The minimum-jerk model~\cite{flash1985coordination} predicts that skilled reaching movements follow fifth-order polynomial trajectories minimizing $\int\|\dddot{x}(t)\|^2 dt$, producing bell-shaped velocity profiles. Signal-dependent noise theory~\cite{harris1998signal} establishes that motor noise scales with control signal magnitude: suboptimal controllers that use larger corrective signals experience proportionally more noise, producing non-smooth trajectories with excess high-frequency energy. These results imply that smoothness is \emph{constitutive} of skilled behavior, not merely correlated with it. SAL and TED directly measure this property.

\paragraph{Filtering can reduce conditional action variance.}
Consider a dataset with fraction $\alpha$ of clean demonstrations (within-group variance $\sigma_c^2$) and $1-\alpha$ noisy ones ($\sigma_n^2 \gg \sigma_c^2$). When both groups target the same conditional mean action $\mu^*(s)$, the noise floor decomposes as:
\begin{equation}
    \bar\sigma^2 = \alpha\,\sigma_c^2 + (1-\alpha)\,\sigma_n^2.
    \label{eq:sigma_mixture}
\end{equation}
In practice, noisy demonstrators may also differ in their mean actions, contributing an additional between-group variance term. In either case, filtering to retain the clean subset can reduce the conditional action variance of the retained data distribution, at the cost of reducing sample size from $N$ to $\alpha N$. This trade-off is beneficial when the noise reduction outweighs the increased estimation error. This condition holds when the noise gap is large relative to the policy complexity, the typical regime of heterogeneous teleoperation data.
More precisely, this benefit is with respect to the retained distribution
$p_{\mathrm{keep}}$: lowering conditional action variance improves the learning
signal on the retained subset itself. Improvements on the target task further
require that filtering preserve sufficient coverage of task-relevant states or
observations; otherwise reduced variance may be offset by reduced support.

\paragraph{Compounding error amplification.}
The classical BC regret bound~\cite{ross2010efficient} establishes:
\begin{equation}
    J(\hat\pi) - J(\pi^*) \le H^2\,\varepsilon,
    \label{eq:compounding}
\end{equation}
where $\varepsilon$ is the expected per-step imitation error under the expert's state distribution and $H$ the task horizon. This bound is tight~\cite{rajaraman2020fundamental}: no offline algorithm can improve the $H^2$ scaling for MSE-based BC. Horizon-independent bounds have been shown under log-loss~\cite{foster2024horizon}, but standard BC uses MSE, where the $H^2$ compounding persists. By Proposition~\ref{prop:noise_floor}, the population MSE loss satisfies $L(\hat\theta) \ge \bar\sigma^2$, so the noise floor constrains how small the per-step error can become. The $H^2$ amplification in~(\ref{eq:compounding}) means that even modest reductions in the irreducible per-step error term can translate into improved rollout performance, especially for long-horizon tasks, under standard BC sensitivity to imitation error.
We view~(\ref{eq:compounding}) as a motivating sensitivity bound rather than as
a theorem that, by itself, proves the empirical gains from smoothness
filtering.

\paragraph{Action chunking amplification.}
Modern architectures predict multi-step action chunks of horizon $H_c$ (typically 8--16). The chunk-level noise floor is:
\begin{equation}
    \bar\sigma^2_{\text{chunk}} = \sum_{h=1}^{H_c} \E_s\bigl[\operatorname{Tr}\bigl(\Var[a_h \mid s]\bigr)\bigr] = H_c \cdot \bar\sigma^2_{\text{step}},
    \label{eq:chunk_amplify}
\end{equation}
where the equality holds when per-step noise is approximately stationary and the chunk loss aggregates errors across the prediction horizon. This suggests that chunked predictors can be more sensitive to inconsistent demonstrations than single-step predictors. For Diffusion Policy~\cite{chi2023diffusionpolicy}, which predicts action sequences rather than single actions, reducing disagreement in the demonstration data can therefore be especially beneficial. Although the diffusion denoising process learns to recover action trajectories from noise, the data-level noise floor $\bar\sigma^2$ persists as irreducible conditional variance, since the score network can only denoise toward the conditional mean when demonstrations disagree.

\paragraph{Quality vs.\ quantity.}
The interplay between noise reduction and sample loss creates a non-trivial optimum. If $k$ of $N$ demonstrations are retained, the expected error approximately satisfies:
\begin{equation}
    \mathcal{E}(k) \approx \underbrace{\bar\sigma^2(k)}_{\text{decreasing}} + \underbrace{\frac{C \cdot d}{kT}}_{\text{increasing}},
    \label{eq:quality_quantity}
\end{equation}
where $\bar\sigma^2(k)$ decreases as noisier demonstrations are removed. The second term is the standard statistical estimation error: for a policy class of effective complexity $d$ fit to $n = kT$ state-action pairs, excess risk scales as $O(d/n)$~\cite{ross2010efficient}. An interior optimum $k^*$ exists when $\bar\sigma^2$ decreases faster than estimation error increases, precisely the regime where \textsc{rinse} operates, as confirmed by our experiments.

\paragraph{SAL and TED as complementary quality measures.}
SAL captures global spectral regularity, excelling at detecting high-frequency jitter in free-space motions. TED captures phase-aware geometric coherence, excelling at detecting spatial artifacts in contact-rich tasks. While raw jerk ($\int\|\dddot{x}\|^2 dt$) is the natural smoothness proxy from motor control theory, it penalizes contact transitions indiscriminately and is sensitive to measurement noise and trajectory duration; SAL's spectral formulation is ${\sim}10\times$ more robust to noise than log dimensionless jerk~\cite{Balasubramanian2012SPARC}, and TED's phase-aware segmentation avoids penalizing contact discontinuities. By segmenting at contact boundaries and measuring smoothness \emph{within} each phase, TED provides a physically meaningful quality signal even for multi-phase tasks.

\section{Experiments}
\label{sec:experiments}

We evaluate \textsc{rinse} in two complementary settings: (1)~as a quality diagnostic that ranks datasets and domains without any policy training (\S\ref{sec:exp-modality}--\ref{sec:exp-remix}), and (2)~as a filter that selects demonstrations for improved policy learning (\S\ref{sec:exp-robomimic}--\ref{sec:exp-libero}).

\subsection{Quality Diagnostic: Data Collection Modality}
\label{sec:exp-modality}

We apply \textsc{rinse} to the RTIS benchmark~\cite{li2025how}, which provides five manipulation tasks under three modalities: kinesthetic teaching, VR teleoperation, and spacemouse teleoperation.

\begin{figure*}[t]
  \centering
  \includegraphics[width=\linewidth]{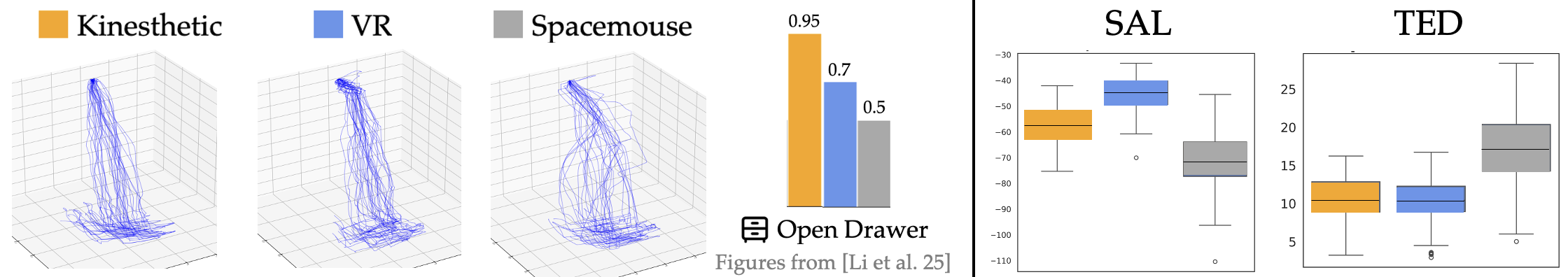}
  \caption{SAL (higher = smoother) and TED (lower = smoother) for Open Drawer from~\cite{li2025how}. Scores correlate with policy success across modalities.}
  \label{fig:rtis}
\end{figure*}

Both metrics consistently rank kinesthetic $>$ VR $>$ spacemouse (\cref{fig:rtis}; per-variant breakdown in \cref{fig:sal-violin}), matching the downstream policy success ordering reported in~\cite{li2025how}. Kinesthetic teaching produces the smoothest trajectories; rate-controlled spacemouse introduces up to $4\times$ larger TED scores. This validates \textsc{rinse} as a diagnostic tool for assessing dataset quality \emph{before} any policy is trained.

\subsection{Quality Diagnostic: Re-Mix Domain Reweighting}
\label{sec:exp-remix}

We apply \textsc{rinse} as soft quality weights within Re-Mix~\cite{hejnaremix}, which performs Group-DRO-style exponentiated updates on domain sampling weights $\alpha_i$ using per-domain excess losses. We modify the estimator by weighting each trajectory's contribution by $w(\tau) = \exp(-\lambda\, q(\tau))$, where $q(\tau)$ is the normalized TED or SAL badness score:
\begin{equation}
    \Delta L_i^{(w)}(\theta) = \frac{\sum_{\tau \in D_i} w(\tau)\,\Delta\ell(\tau;\theta)}{\sum_{\tau \in D_i} w(\tau)}.
    \label{eq:remix_weighted}
\end{equation}
We train on 7 Open X-Embodiment domains for 100k steps (details in \hyperref[sec:remix-details]{Appendix~f}). \Cref{tab:remix} reports final $\alpha$ values after domain allocation under both z-score and rank normalization.

\begin{table}[t]
  \centering
  \footnotesize
  \caption{Re-Mix $\alpha$ values (\%) after domain allocation (mean over 3 seeds). Top: z-score normalization; bottom: rank normalization. \colorbox{cellgreen}{\strut Green}: $>$25\% relative increase from uniform; \colorbox{cellred}{\strut red}: $>$25\% relative decrease.}
  \label{tab:remix}
  \vspace{.3em}
  \setlength{\tabcolsep}{2pt}
  \begin{tabular}{lcccc}
    \toprule
    \textbf{Domain} & \textbf{Uniform} & \textbf{Re-Mix} & \textbf{Re-Mix+TED} & \textbf{Re-Mix+SAL}\\
    \midrule
    \multicolumn{5}{l}{\textit{Z-score normalization} ($\rho \in [0.893, 0.964]$, cos $\in [0.989, 0.999]$)}\\
    \midrule
    autolab\_ur5     & 9.3 & \cellcolor{cellred}3.1 & \cellcolor{cellred}5.8 & \cellcolor{cellred}3.1\\
    cable\_routing   & 3.9 & \cellcolor{cellred}0.7 & \cellcolor{cellred}0.9 & \cellcolor{cellred}0.9\\
    jaco\_play       & 7.4 & \cellcolor{cellred}1.4 & \cellcolor{cellred}2.0 & \cellcolor{cellred}1.7\\
    roboturk         & 17.9 & \cellcolor{cellred}3.9 & \cellcolor{cellred}4.8 & \cellcolor{cellred}4.4\\
    taco\_play       & 22.6 & \cellcolor{cellred}5.0 & \cellcolor{cellred}11.4 & \cellcolor{cellred}7.7\\
    toto             & 31.5 & \cellcolor{cellgreen}\textbf{82.6} & \cellcolor{cellgreen}66.5 & \cellcolor{cellgreen}77.6\\
    viola            & 7.4 & \cellcolor{cellred}3.2 & 8.6 & \cellcolor{cellred}4.5\\
    \midrule
    \multicolumn{5}{l}{\textit{Rank normalization} ($\rho = 1.000$, cos $\in [0.997, 1.000]$)}\\
    \midrule
    autolab\_ur5     & 9.3 & \cellcolor{cellred}1.8 & \cellcolor{cellred}3.3 & \cellcolor{cellred}1.6\\
    cable\_routing   & 3.9 & \cellcolor{cellred}0.4 & \cellcolor{cellred}1.0 & \cellcolor{cellred}0.5\\
    jaco\_play       & 7.4 & \cellcolor{cellred}1.0 & \cellcolor{cellred}2.0 & \cellcolor{cellred}1.2\\
    roboturk         & 17.9 & \cellcolor{cellred}2.3 & \cellcolor{cellred}4.7 & \cellcolor{cellred}2.5\\
    taco\_play       & 22.6 & \cellcolor{cellred}3.1 & \cellcolor{cellred}8.1 & \cellcolor{cellred}4.0\\
    toto             & 31.5 & \cellcolor{cellgreen}\textbf{90.5} & \cellcolor{cellgreen}79.2 & \cellcolor{cellgreen}89.1\\
    viola            & 7.4 & \cellcolor{cellred}0.8 & \cellcolor{cellred}1.7 & \cellcolor{cellred}1.0\\
    \bottomrule
  \end{tabular}
\end{table}

\textsc{rinse}-weighted allocations are highly correlated with Re-Mix base allocations (Spearman $\rho \ge 0.89$, cosine $\ge 0.989$). Since Re-Mix allocations are validated to improve downstream policy performance~\cite{hejnaremix}, this indicates that \textsc{rinse} scores, computed offline without any policy training, recover a domain-quality signal known to matter for learning. All variants concentrate weight on the \emph{toto} domain, which Re-Mix identifies as having the highest marginal training value.

\subsection{Filtering: RoboMimic Benchmarks}
\label{sec:exp-robomimic}

We apply smoothness-based filtering to three RoboMimic benchmarks~\cite{robomimic2021} and train Diffusion Policy (Transformer backbone)~\cite{chi2023diffusionpolicy} on the filtered subsets. For each task, we select a budget of $K$ demonstrations ranked by TED or SAL, and compare against training on the full dataset. \Cref{tab:rbm-main} reports mean success rates over 3 seeds after $\ge$100k gradient steps. For bimanual Transport, smoothness scores are averaged over both arms.

\begin{table}[t]
  \centering
  \footnotesize
  \caption{RoboMimic filtering results (\% success). $K$ demos selected from $|\mathcal{D}|$ by TED or SAL vs.\ full data. mh = multi-human, ph = proficient-human, b--b = better--better.}
  \label{tab:rbm-main}
  \vspace{.3em}
  \setlength{\tabcolsep}{3.5pt}
  \begin{tabular}{lcccccc}
    \toprule
    \textbf{Task} & \textbf{Src} & $|\mathcal{D}|$ & $K$ &
    \textbf{TED} & \textbf{SAL} & \textbf{Full}\\
    \midrule
    Square       & mh     & 300 & 50  & 75 & 72 & \textbf{76} \\
                 & better & 100 & 50  & \textbf{82} & \textbf{82} & 78 \\
    \midrule
    Tool Hang    & ph     & 200 & 100 & 58 & \textbf{71} & 63 \\
    \midrule
    Transport    & mh     & 300 & 50  & 46 & \textbf{55} & 39 \\
                 & b--b   &  50 & 25  & 36 & \textbf{54} & 42 \\
    \bottomrule
  \end{tabular}
\end{table}

On most tasks, either TED or SAL filtering matches or outperforms the full-data baseline despite using as few as one-sixth of the demonstrations. The gains are most pronounced on Transport, a complex bimanual task with substantial free-space motion where SAL filtering achieves a \textbf{16\% improvement} using 50 of 300 demos. On Tool Hang, SAL with half the demonstrations exceeds full-data performance by \textbf{8\%}. Within the ``better'' labeled subsets, both metrics further discriminate quality: filtered top-50 reach 82\% versus 78\% for the full ``better'' set on Square.

The metrics are complementary: SAL excels on tasks with prominent free-space jitter (Transport), while TED captures spatial artifacts in precision contact tasks (Square), consistent with their design (\S\ref{sec:theory}). Convergence speed also improves with filtering (\hyperref[sec:convergence-notes]{Appendix~d}).

\subsection{Filtering: Real-World Experiments}
\label{sec:exp-realworld}

\begin{figure}[t]
  \centering
  \includegraphics[width=0.85\columnwidth]{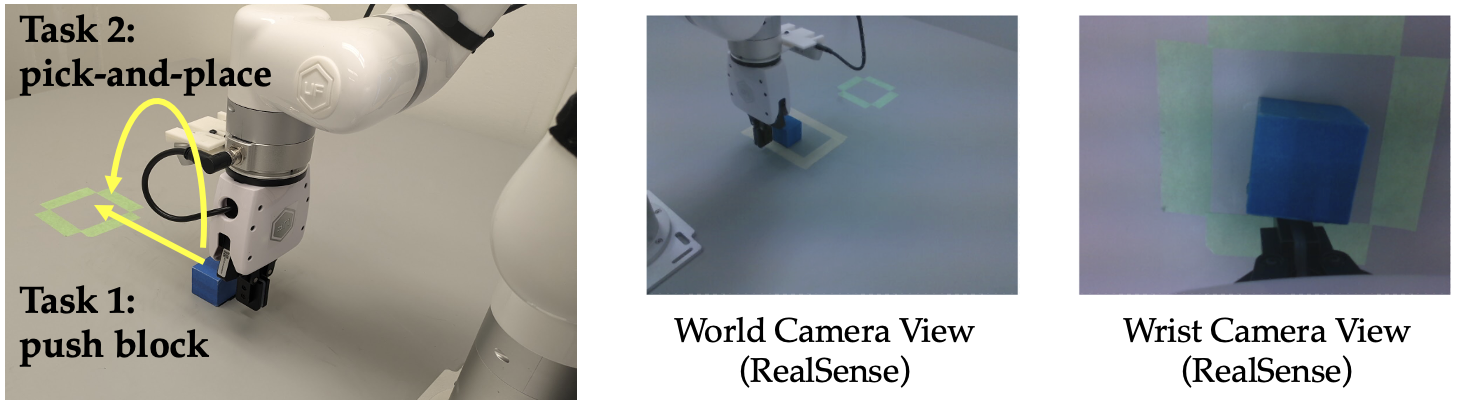}
  \caption{Real-world setup: U-Factory xArm-7 with two RealSense cameras (wrist and side-mounted) for visual observations.}
  \label{fig:real_world_experiments}
\end{figure}

We evaluate on a U-Factory xArm 7-DoF robot with two RealSense cameras (\cref{fig:real_world_experiments}). We evaluate two tasks: \textbf{Push Block} and \textbf{Pick and Place}. All policies use Diffusion Policy (1D Conv UNet backbone) evaluated at epoch 50 over 25 randomized trials.

\begin{table}[t]
  \centering
  \footnotesize
  \caption{Real-world filtering results (\% success). Rnd = average of 3 random subsets of equal size (see \hyperref[sec:random-comparison]{Appendix~c}).}
  \label{tab:real-results}
  \vspace{.3em}
  \setlength{\tabcolsep}{3pt}
  \begin{tabular}{lcccccc}
    \toprule
    \textbf{Task} & $|\mathcal{D}|$ & $K$ &
    \textbf{Rnd} & \textbf{TED} & \textbf{SAL} & \textbf{Full}\\
    \midrule
    Push Block  & 200 & 150  & 64 & 80 & 72 & 68 \\
    Push Block  & 200 & 100  & 59 & \textbf{88} & 84 & 68 \\
    Push Block  & 200 &  50  & 52 & 84 & 84 & 68 \\
    \midrule
    Pick\,\&\,Place & 80 & 40  & 55 & \textbf{76} & 68 & 76 \\
    \bottomrule
  \end{tabular}
\end{table}

TED filtering achieves \textbf{20\% higher success} on Push Block with half the data (\cref{tab:real-results}). At every subset size, TED outperforms random subsets by \textbf{16--32 percentage points}; SAL outperforms by \textbf{8--32 pp}, confirming that \textsc{rinse}'s ranking, not subset size, drives improvement. On Pick and Place ($|\mathcal{D}|{=}80$), TED matches full-data performance, consistent with a smaller, relatively clean dataset where filtering has less room to improve. SAL scores also correlate with operator-reported task difficulty (\hyperref[appendix:operator-difficulty]{Appendix~g}).

\subsection{Filtering: Retrieval-Augmented Learning}
\label{sec:exp-libero}

We integrate \textsc{rinse} into the STRAP retrieval framework~\cite{memmel2024strap} on LIBERO-10~\cite{liu2024libero}, a suite of 10 language-conditioned manipulation tasks. STRAP segments target demonstrations into sub-trajectories and retrieves similar segments from a prior pool via subsequence DTW (SDTW) over DINOv2 visual embeddings. Our variant takes the top $R{=}400$ SDTW candidates per query, re-ranks by TED or SAL, and retains the top $K{=}200$ (details in \hyperref[appendix:libero-details]{Appendix~e}). All methods train a language-conditioned BC Transformer; \cref{tab:libero} reports the mean over 3 seeds of each seed's best checkpoint under the common evaluation protocol.

\begin{table}[t]
  \centering
  \scriptsize
  \caption{LIBERO-10 success rates (\%). \textsc{rinse} re-ranking of STRAP-retrieved sub-trajectories improves aggregate performance. Best per-task in \textbf{bold}.}
  \label{tab:libero}
  \vspace{.3em}
  \setlength{\tabcolsep}{1.8pt}
  \begin{tabular}{l*{10}{c}c}
    \toprule
    & \rotatebox{75}{\scriptsize Mug-MW} & \rotatebox{75}{\scriptsize Moka$^2$} & \rotatebox{75}{\scriptsize Soup-S} & \rotatebox{75}{\scriptsize CCB} & \rotatebox{75}{\scriptsize Mug-P} & \rotatebox{75}{\scriptsize Stove-M} & \rotatebox{75}{\scriptsize Bowl-C} & \rotatebox{75}{\scriptsize Soup-C} & \rotatebox{75}{\scriptsize Mug$^2$} & \rotatebox{75}{\scriptsize Book-C} & \textbf{Avg}\\
    \midrule
    STRAP & \textbf{16} & 0 & 40 & 16 & 22 & \textbf{84} & 76 & 24 & 48 & 80 & 40.6\\
    +TED  & 14 & 0 & 32 & \textbf{48} & \textbf{38} & 72 & \textbf{92} & \textbf{34} & 48 & 84 & \textbf{46.2}\\
    +SAL  & 14 & 0 & \textbf{44} & 8 & 34 & 68 & 84 & 40 & \textbf{58} & \textbf{100} & 45.0\\
    \bottomrule
  \end{tabular}
\end{table}

TED+STRAP and SAL+STRAP improve aggregate success by \textbf{5.6\%} and \textbf{4.4\%} respectively over pure STRAP retrieval, although the per-task outcomes remain mixed. Some tasks regress relative to pure SDTW retrieval, while others improve substantially, indicating that smoothness-aware re-ranking is most helpful when the retrieved pool contains quality variation that SDTW alone does not resolve. The metrics are complementary: TED excels on contact-sensitive tasks (Cream-Cheese-Butter: $+$32\%), while SAL excels on tasks with free-space motion (Book-Caddy: $+$20\%). This demonstrates that \textsc{rinse} scores are effective as a secondary ranking signal within retrieval pipelines, not only as a standalone filter.

\section{Conclusion}
\label{sec:conclusion}

We presented \textsc{rinse}, a lightweight framework for evaluating and filtering demonstration quality based on trajectory smoothness. By adapting Spectral Arc Length (SPARC; SAL) from rehabilitation science and introducing TED, both grounded in motor control theory, \textsc{rinse} provides a principled quality signal that is broadly compatible with different policy architectures, with TED additionally using a phase-boundary/contact signal. Our theoretical analysis shows that smoothness filtering can reduce the conditional action variance of the retained data distribution, with downstream effects that can be amplified by action chunking and compounding error. Experiments across RoboMimic, real-world manipulation, LIBERO-10 retrieval augmentation, and Re-Mix domain reweighting show improvements across filtering, retrieval, and reweighting settings, with the largest gains in noisy or heterogeneous data regimes.

\textbf{Limitations.}
Smoothness is necessary but not sufficient for quality: smooth yet suboptimal trajectories may exist. \textsc{rinse} assumes pre-filtering for task success and operates on proprioceptive signals, though it could be combined with state-coverage metrics; all policies in our experiments use image observations, showing proprioceptive filtering complements visual policy learning. TED's phase-aware segmentation mitigates but does not eliminate the risk of penalizing task-necessary abrupt motions. Future work should explore integrating smoothness with diversity-aware metrics and real-time quality feedback during data collection.

\section*{APPENDIX}

\begin{figure*}[t]
  \centering
  \includegraphics[width=\linewidth]{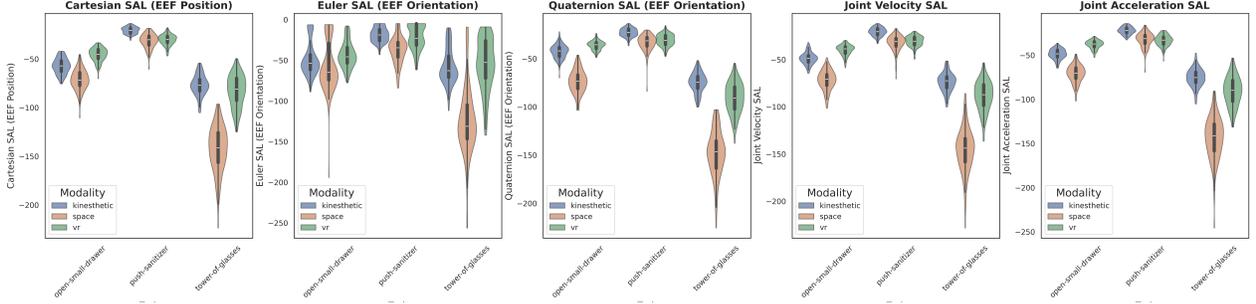}
  \caption{SAL violin plots across five metric variants and three modalities for three tasks from~\cite{li2025how}. Kinesthetic teaching consistently produces higher (smoother) SAL.}
  \label{fig:sal-violin}
\end{figure*}

\paragraph{SAL hyperparameters.}
\label{appendix:sal-hyperparams}
The sampling interval $\Delta t$ matches the control frequency (e.g., $\Delta t = 0.05\,\text{s}$ at 20\,Hz), setting $f_N = 1/(2\Delta t)$. All frequencies from DC to $f_N$ are included without additional pre-filtering. A regularization constant $\varepsilon = 10^{-5}$ prevents numerical issues. SAL requires no task-specific tuning; for high-frequency tasks (e.g., hammering), a higher control rate ensures task-relevant oscillations fall below the Nyquist cutoff.

\paragraph{TED algorithm and hyperparameters.}
\label{appendix:ted}
Algorithm~\ref{alg:ted} describes TED's greedy refine-and-keep-best heuristic with corridor fallback. Each contact phase is split into dense boundary regions and sparse interior ($r = 0.2$). Dense and sparse regions use separate corridor widths and initialization settings: $\varepsilon_{\text{dense}} = 0.10$, $\varepsilon_{\text{sparse}} = 0.05$, $K_{0,\text{dense}} = 10$, $K_{0,\text{sparse}} = 12$, and $M_{\text{ref}} = 20$. Orientations are smoothed via sliding-window geodesic averaging ($w = 4$ timesteps). DTW orientation weight $w_{\mathrm{ori}} = 0.25$. All values are tunable but were set once and used across all experiments without task-specific tuning.

\begin{algorithm}[t]
\caption{Greedy B\'ezier envelope with corridor fallback (TED)}
\label{alg:ted}
\footnotesize
\begin{algorithmic}[1]
\Require Raw positions $\{x_i\}_{i=0}^{n-1}$, corridor width $\epsilon$, area budget $A_{\mathrm{budget}}$, refinement threshold $\tau_{\mathrm{ref}}$, max iterations $M_{\mathrm{ref}}$, initial control-point indices $I_0$, numerical tolerance $\epsilon_{\mathrm{num}}$
\State $c \gets \{x_j \mid j \in I_0\}$
\State $B_{\mathrm{best}} \gets \varnothing$, $s_{\mathrm{best}} \gets +\infty$
\For{$m = 1$ to $M_{\mathrm{ref}}$}
    \State $B \gets \textsc{BezierFit}(c)$
    \State $d_i \gets \|x_i - B(i/(n-1))\|,\;\; i=0,\dots,n-1$
    \State $A_{\mathrm{cur}} \gets \frac{1}{n-1}\sum_{i=0}^{n-1} d_i$
    \State $P_{\mathrm{cur}} \gets \sum_{i=0}^{n-1}\left[\max(0, d_i-\epsilon)\right]^2$
    \State $s_{\mathrm{cur}} \gets A_{\mathrm{cur}} + P_{\mathrm{cur}}$
    \If{$s_{\mathrm{cur}} < s_{\mathrm{best}}$}
        \State $B_{\mathrm{best}} \gets B$, $s_{\mathrm{best}} \gets s_{\mathrm{cur}}$
    \EndIf
    \If{$A_{\mathrm{cur}} \le A_{\mathrm{budget}}$ \textbf{and} $P_{\mathrm{cur}} \le \epsilon_{\mathrm{num}}$}
        \State \textbf{break}
    \EndIf
    \State $j^\star \gets \arg\max_i d_i$
    \If{$d_{j^\star} \le \tau_{\mathrm{ref}}$}
        \State \textbf{break}
    \EndIf
    \State $c \gets c \cup \{x_{j^\star}\}$;\quad sort $c$ in temporal order
\EndFor
\State $B_{\mathrm{corr}} \gets \textsc{CorridorFallback}(\{x_i\}_{i=0}^{n-1}, \epsilon, A_{\mathrm{budget}})$
\State $d^{\mathrm{corr}}_i \gets \|x_i - B_{\mathrm{corr}}(i/(n-1))\|$
\State $s_{\mathrm{corr}} \gets \frac{1}{n-1}\sum_i d^{\mathrm{corr}}_i + \sum_i\left[\max(0, d^{\mathrm{corr}}_i-\epsilon)\right]^2$
\If{$s_{\mathrm{corr}} < s_{\mathrm{best}}$}
    \State \Return $\{B_{\mathrm{corr}}(i/(n-1))\}_{i=0}^{n-1}$
\Else
    \State \Return $\{B_{\mathrm{best}}(i/(n-1))\}_{i=0}^{n-1}$
\EndIf
\end{algorithmic}
\end{algorithm}

\noindent Here $A_{\mathrm{cur}}$ is an interval-normalized residual surrogate (not a physical area), and $\tau_{\mathrm{ref}}$ controls when further control-point insertion stops; in our implementation, we set $\tau_{\mathrm{ref}} = A_{\mathrm{budget}}$.

\paragraph{Random subset comparison.}
\label{sec:random-comparison}
To confirm gains are due to quality selection, we trained Diffusion Policy (1D Conv UNet; batch 64, AdamW lr $10^{-4}$, 100 DDPM / 16 DDIM steps, ResNet-18) on random $K$-sized subsets (3 seeds). \Cref{tab:random-seeds} shows TED and SAL outperform random by \textbf{8--40 pp}.

\begin{table}[h]
  \centering
  \footnotesize
  \caption{Random subsets vs.\ \textsc{rinse} filtering (\% success).}
  \label{tab:random-seeds}
  \vspace{.3em}
  \setlength{\tabcolsep}{3pt}
  \begin{tabular}{lcccccc}
    \toprule
    \textbf{Task} & $K$ & \textbf{Rnd-42} & \textbf{Rnd-371} & \textbf{Rnd-970} & \textbf{TED} & \textbf{SAL}\\
    \midrule
    Push Block & 50  & 52 & 60 & 44 & \textbf{84} & \textbf{84} \\
    Push Block & 100 & 60 & 56 & 60 & \textbf{88} & \textbf{84} \\
    Push Block & 150 & 64 & 68 & 60 & \textbf{80} & 72 \\
    \midrule
    Pick\,\&\,Place & 40 & 56 & 48 & 60 & \textbf{76} & 68 \\
    \bottomrule
  \end{tabular}
\end{table}

\paragraph{Convergence analysis.}
\label{sec:convergence-notes}
\Cref{tab:convergence} summarizes convergence and top-vs-bottom ablations. RoboMimic uses Diffusion Policy (Transformer)~\cite{chi2023diffusionpolicy} (obs.\ 2, action 8, pred.\ 16). Smoothness-ranked subsets achieve \textbf{10--35\%} gains even \emph{within} high-quality labels.

\begin{table}[h]
  \centering
  \footnotesize
  \caption{Convergence and ablation results (\% success).}
  \label{tab:convergence}
  \vspace{.3em}
  \setlength{\tabcolsep}{3pt}
  \begin{tabular}{llcc}
    \toprule
    \textbf{Task (Source)} & \textbf{Metric} & \textbf{Top $K$} & \textbf{Bottom $K$}\\
    \midrule
    Square (better, 100) & TED, $K{=}50$ & \textbf{82\%} & 47\% \\
    Tool Hang (ph, 200) & SAL, $K{=}100$ & \textbf{71\%} & $\le$61\% \\
    Transport (b--b, 50) & SAL, $K{=}25$ & \textbf{54\%} & $\le$44\% \\
    \bottomrule
  \end{tabular}
  \vspace{.3em}

  \raggedright\footnotesize\textit{Convergence:} On Square (mh, 300), the filtered top-50 reaches 75\% in $\approx$40k steps; the full 300-demo model needs $\approx$450k steps ($>$\textbf{10$\times$ speed-up}).
\end{table}

\paragraph{LIBERO-10 training details.}
\label{appendix:libero-details}
STRAP retrieves candidates via SDTW over DINOv2~(ViT-B/14) embeddings (L2). We take top $R{=}400$ per query, re-rank by TED or SAL, retain top $K{=}200$ ($\lfloor K/Q \rfloor$ per query, remainder round-robin). Policy: BC Transformer, context 5, dim 256, 8 layers, 4 heads, GMM head (5 modes), ResNet-18 + SpatialSoftmax. Training: 300 epochs, batch 32, AdamW lr $10^{-4}$, eval every 20 epochs (25 episodes), 5 demos/task.

\paragraph{Re-Mix integration details.}
\label{sec:remix-details}
We weight each trajectory by $w(\tau) = \exp(-\lambda\, q(\tau))$ where $q(\tau)$ is a normalized badness score. Z-score: $q_i^{(z)} = (q_i - \mu_i)/\sigma_i$; rank: $q_i^{(\mathrm{rank})} = (\mathrm{rank}_i - 1)/(N_i - 1)$. Lambda calibrated so $w_{90}/w_{10} = 10$: z-score $\lambda_{\text{TED}} = 0.883$, $\lambda_{\text{SAL}} = 0.965$; rank $\lambda = 2.876$. Training: 7 OXE domains, 100k steps, batch 8, $\eta = 0.2$, $\gamma = 0.05$, 3 seeds. RTIS: scores only (no training).

\paragraph{Operator difficulty analysis.}
\label{appendix:operator-difficulty}
\Cref{fig:operator-difficulty} shows SAL distributions stratified by operator-reported difficulty. Trajectories rated easier consistently have higher SAL, confirming \textsc{rinse} scores align with human-perceived execution quality.

\begin{figure}[h]
  \centering
  \includegraphics[width=0.9\columnwidth]{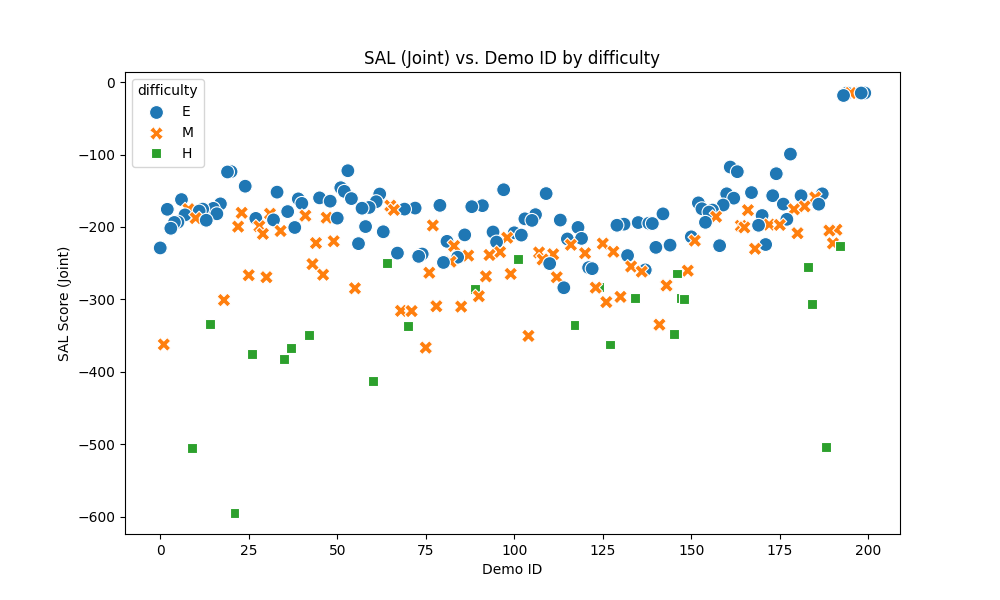}
  \includegraphics[width=0.9\columnwidth]{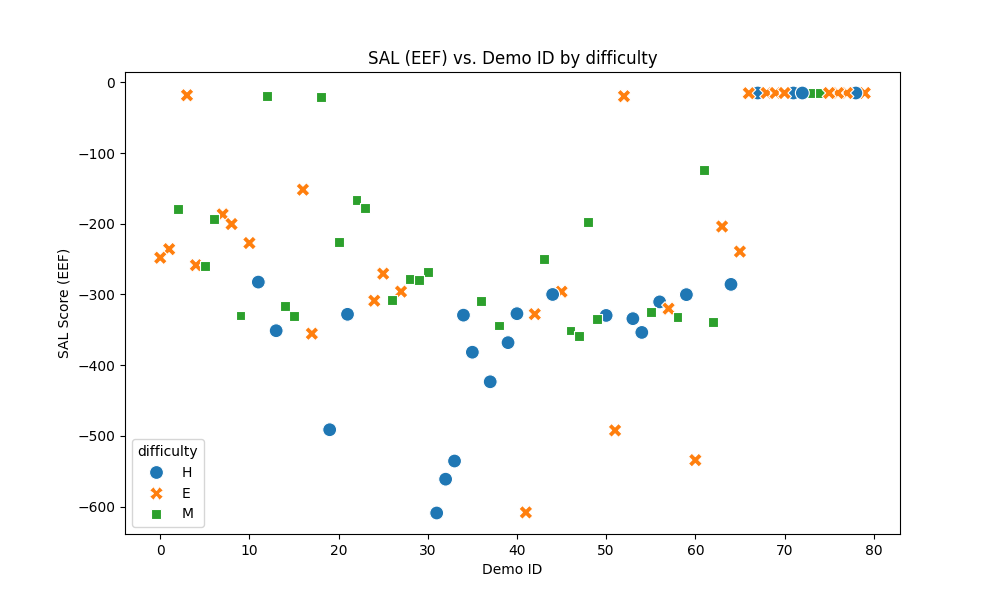}
  \caption{SAL distributions stratified by operator-reported difficulty for Push Block (top) and Pick and Place (bottom). Easier trajectories have higher SAL (smoother).}
  \label{fig:operator-difficulty}
\end{figure}

All experiments use NVIDIA L40S GPUs.

\bibliographystyle{format_files/IEEEtranBST/IEEEtran}
\bibliography{references}

\end{document}